\documentclass[11pt]{article}

\usepackage{fullpage}

\usepackage{mathptmx}
\usepackage{amsmath}
\usepackage{amsfonts}
\usepackage{amssymb}
\usepackage{amsthm}

\usepackage{graphicx}
\usepackage{wrapfig} 
\usepackage[numbers]{natbib}

\usepackage{algorithm}
\usepackage{algorithmic}

\usepackage[draft]{hyperref}
\usepackage{euscript}
\usepackage{xspace}
\usepackage{paralist}
\usepackage{subfig} 
\usepackage{multicol} 
\usepackage{tabularx}

\usepackage[draft,inline,nomargin,index]{fixme}

\theoremstyle{plain}

\newtheorem{lemma}{Lemma}[section]

\numberwithin{equation}{section}

\fxsetup{theme=colorsig,mode=multiuser,inlineface=\sffamily,envface=\itshape}
\FXRegisterAuthor{sv}{asv}{Suresh}
\FXRegisterAuthor{pr}{apr}{Parasaran} 
\FXRegisterAuthor{jm}{ajm}{John} 
\FXRegisterAuthor{as}{aas}{Avishek}

\renewcommand{\vec}[1]{\ensuremath{\mathbf{#1}}}
\renewcommand{\b}[1]{\ensuremath{\mathbb{#1}}}

\newcommand{\eps}{\varepsilon}
\newcommand{\diag}{\mathop{\rm diag}}
\newcommand{\trace}{\mathop{\rm tr}}

\newcommand{\argmax}{\operatorname*{arg\,max}}
\newcommand{\balpha}{\boldsymbol{\alpha}}
\newcommand{\bmu}{\boldsymbol{\mu}}

\newcommand{\vone}{\vec{1}}
\newcommand{\innerprod}[2]{\left\langle #1, #2 \right\rangle}

\newcommand{\uni}{\textsc{\mbox{Uniform}}\xspace}
\newcommand{\llplus}{\textsc{\mbox{LibLinear+}}\xspace}
\newcommand{\sdp}{\textsc{SdpMKL}\xspace}
\newcommand{\silp}{\textsc{SILP}\xspace}
\newcommand{\smkl}{\textsc{SimpleMKL}\xspace}
\newcommand{\lmkl}{\textsc{LevelMKL}\xspace}
\newcommand{\glmkl}{\textsc{GroupMKL}\xspace}
\newcommand{\smomkl}{\textsc{SMOMKL}\xspace}

\newcommand{\mwumkl}{\textsc{\mbox{MWUMKL}}\xspace}

\newcommand{\iono}{\textit{Iono}\xspace}
\newcommand{\cancr}{\textit{Breast Cancer}\xspace}
\newcommand{\pima}{\textit{Pima}\xspace}
\newcommand{\sonar}{\textit{Sonar}\xspace}
\newcommand{\heart}{\textit{Heart}\xspace}
\newcommand{\vote}{\textit{Vote}\xspace}
\newcommand{\wdbc}{\textit{WDBC}\xspace}
\newcommand{\wpbc}{\textit{WPBC}\xspace}
\newcommand{\mush}{\textit{Mushroom}\xspace}
\newcommand{\adult}{\textit{Adult}\xspace}

\newcommand{\web}{\textit{Web}\xspace}

\newcommand{\cod}{\textit{CodRna}\xspace}

\title{A Geometric Algorithm for Scalable Multiple Kernel Learning}
\author{John Moeller\\moeller@cs.utah.edu \and Parasaran Raman\\praman@yahoo-inc.com \and Avishek Saha\\avishek2@yahoo-inc.com \and Suresh Venkatasubramanian\\suresh@cs.utah.edu}

\begin{document}

\maketitle

\begin{abstract}
We present a geometric formulation of the Multiple Kernel Learning (MKL) problem. 
To do so, we reinterpret the problem of learning kernel weights as searching for a kernel that maximizes the minimum (kernel) distance between two convex polytopes. 
This interpretation combined with novel structural insights from our geometric formulation allows us to reduce the MKL problem to a simple optimization routine that yields provable convergence as well as quality guarantees. 
As a result our method scales efficiently to much larger data sets than most prior methods can handle. 
Empirical evaluation on eleven datasets shows that we are significantly faster and even compare favorably with a uniform unweighted combination of kernels.

\end{abstract}

\section{Introduction}
\label{sec:intro}
Multiple kernel learning is a principled alternative to choosing kernels (or kernel weights) and has been successfully applied to a wide variety of learning tasks and domains \cite{Lanckriet04MKLSDP,Bach04MKL,Argyriou:2006:DAK:1143844.1143850,Zien:2007:MMK:1273496.1273646,DBLP:conf/nips/CristianiniSEK01,Ye:2007:DKR:1273496.1273634,Pavlidis_wn:gene,Sonnenburg06MKLSILP}.
Pioneering work by \citet{Lanckriet04MKLSDP} jointly optimizes the Support Vector Machine (SVM) task and the choice of kernels by exploiting convex optimization at the heart of both problems.
Although theoretically elegant, this approach requires repeated invocations of semidefinite solvers. 
Other existing methods~\cite{Sonnenburg06MKLSILP,Lanckriet04MKLSDP,Rakotomamonjy07MKL,DBLP:conf/nips/XuJKL08,DBLP:conf/icml/XuJYKL10}, albeit accurate, are slow and have large memory footprints.

In this paper, we present an alternate \emph{geometric} perspective on the MKL problem. 
The starting point for our approach is to view the MKL problem as an optimization of kernel distances over convex polytopes.
The ensuing formulation is a Quadratically Constrainted Quadratic Program (QCQP) which we solve using a novel variant of the Matrix Multiplicative Weight Update (MMWU) method of \citet{Arora:2007:CPA:1250790.1250823}; a primal-dual combinatorial algorithm for solving Semidefinite Programs (SDP) and QCQPs. 
While the MMWU approach in its generic form does not yield an efficient solution for our problem, we show that a careful geometric reexamination of the primal-dual algorithm reveals a simple alternating optimization with extremely light-weight update steps. 
This algorithm can be described as simply as: ``find a few violating support vectors with respect to the current kernel estimate, and reweight the kernels based on these support vectors''. 

Our approach 
\begin{inparaenum}[(a)] 
\item does not require commercial cone or SDP solvers, 
\item does not make explicit calls to SVM libraries (unlike alternating optimization based methods),
\item provably converges in a fixed number of iterations, and 
\item has an extremely light memory footprint.
\end{inparaenum}
Moreover, our focus is on optimizing MKL on a \emph{single machine}. 
Existing techniques~\cite{Sonnenburg06MKLSILP} that use careful engineering to parallelize MKL optimizations in order to scale can be viewed as complementary to our work.
Indeed, our future work is focused on adding parallel components to our already fast optimization method.

A detailed evaluation on eleven datasets shows that our proposed algorithm 
\begin{inparaenum}[(a)]
\item is fast, even as the data size increases beyond a few thousand points, 
\item compares favorably with LibLinear~\cite{REF08a} after Nystr\"om kernel approximations are applied as feature transformations, and 
\item compares favorably with the \emph{uniform} heuristic that merely averages all kernels without searching for an optimal combination. 
\end{inparaenum}
As has been noted~\cite{uniform}, the \emph{uniform} heuristic is a strong
baseline for the evaluation of MKL methods. We use LibLinear with Nystr\"om approximations (\llplus) as an additional scalable baseline, and we are able to beat both these baselines when both $m$ and $n$ are significantly large.

\section{Related Work}
\label{sec:related}

In practice, since the space of all kernels can be unwieldy, many methods operate by fixing a base set of kernels and determining an optimal (conic) combination. 
An early approach (\uni) eliminated the search and simply used an equal-weight sum of kernel functions~\cite{Pavlidis_wn:gene}. 
In their seminal work, \citet{Lanckriet04MKLSDP} proposed to simultaneously train an SVM as well as learn a convex combination of kernel functions. 
The key contribution was to frame the learning problem as an optimization over positive semidefinite kernel matrices which in turn reduces to a QCQP.
. 
Soon after, \citet{Bach04MKL} proposed a block-norm regularization method based on \emph{second order cone programming} (SOCP). 

For efficiency, researchers started using alternating optimization methods that alternate between updating the classifier parameters and the kernel weights. 
\citet{Sonnenburg06MKLSILP} modeled the MKL objective as a cutting plane problem and solved for kernel weights using Semi-Infinite Linear Programming (\silp) techniques.
\citet{Rakotomamonjy07MKL} used sub-gradient descent based methods to solve the MKL problem. 
An improved level set based method
that combines cutting plane models with projection to level sets was proposed by \citet{DBLP:conf/nips/XuJKL08}. 
\citet{DBLP:conf/icml/XuJYKL10}
also derived a variant of the equivalence between group LASSO and the MKL formulation that leads to closed-form updates for kernel weights.
However, as pointed out in~\cite{uniform}, most of these methods do not compare favorably (both in accuracy as well as speed) even with the simple \emph{uniform} heuristic.

Other works in MKL literature study the use of different kernel families, such as Gaussian families~\cite{Micchelli:2005:LKF:1046920.1088710}, hyperkernels~\cite{DBLP:journals/jmlr/OngSW05} and non-linear families~\cite{DBLP:conf/icml/VarmaB09,DBLP:conf/nips/CortesMR09}. 
Regularization based on the $\ell_2$-norm~\cite{DBLP:journals/jmlr/KloftBSZ11} and $\ell_p$-norm~\cite{NIPS2009_0879,DBLP:conf/nips/VishwanathansAV10} have also been introduced. 
In addition, stochastic gradient descent based online algorithms for MKL have been studied in~\cite{DBLP:conf/icml/OrabonaL11}.
Another work by \citet{DBLP:conf/kdd/JainVV12} discusses a scalable MKL algorithm for dynamic kernels. We briefly discuss and compare with this work when presenting empirical results (Section~\ref{sec:results}).

In \emph{two-stage kernel learning}, instead of combining the optimization of kernel weights as well as that of the best hypothesis in a single cost function, the goal is to learn the kernel weights in the first stage and then use it to learn the best classifier in the second stage.
Recent two-stage approaches seem to do well in terms of accuracy -- such as \citet{DBLP:conf/icml/CortesMR10}, who optimize the kernel weights in the first stage and learn a standard SVM in the second stage, and \citet{bib-kumar-icml-12}, who train on meta-examples derived from kernel combinations on the ground examples.
In~\citet{DBLP:conf/icml/CortesMR10}, the authors observe that their algorithm reduces to solving a meta-SVM which can be solved using standard off-the-shelf SVM tools such as LibSVM.
However, despite being highly efficient on few examples, LibSVM is very inefficient on more than a few thousand examples due to quadratic scaling~\cite{libsvm}.
As for~\citet{bib-kumar-icml-12}, the construction of meta-examples scales quadratically in the number of samples and so their algorithm may not scale well past the small datasets evaluated in their work.

Interestingly, our proposed \mwumkl can easily be run as a \emph{single}-kernel algorithm.
We can then apply our scalability to the two-stage algorithm of~\cite{DBLP:conf/icml/CortesMR10}, allowing it not to be limited by the same constraints as LibSVM (which scales quadratically or worse in the number of examples~\cite{libsvm}).

\section{Background} 
\label{sec:back}

\paragraph{Notation.}
We will denote vectors by boldface lower case letters like $\vec{z}$, and matrices by bold uppercase letters $\vec{M}$. 

\begin{table}[!htbp]
  \footnotesize
  \centering
  \begin{tabular}{cl}
    $\vec{0}$ & zero vector or matrix
    \\$\vec{1}$ & all-ones vector or matrix
    \\$\vec{M} \succeq 0$ & $\vec{M}$ is positive semidefinite
    \\$\vec{A} \bullet \vec{B}$ & $\text{Tr}(\vec{A}\vec{B}) = \sum_{i,j}A_{ij}B_{ij}$
    \\$\diag(\vec{a})$ & The diagonal matrix $\vec A$ such that $A_{ii} = a_i$
  \end{tabular}
\end{table}

\paragraph{Modeling the geometry of SVM.}
Suppose that $\vec{X} \in \b{R}^{n \times d}$ is a collection of $n$ training samples in a $d$-dimensional vector space (the rows $\vec{x}_1, \vec{x}_2, \ldots, \vec{x}_n$ are the points). 
Also, $\vec{y} = (y_1, y_2, \ldots, y_n) \in \{ -1, +1 \}^n$ are the binary class labels for the data points in $\vec{X}$. 
Let $\vec{X}_+\subset\vec{X}$ denote the rows corresponding to the positive entries of $\vec{y}$, and likewise $\vec{X}_-\subset\vec{X}$ for the negative entries. 

\begin{figure}[htbp]
  \centering
  \includegraphics[width=.8\columnwidth]{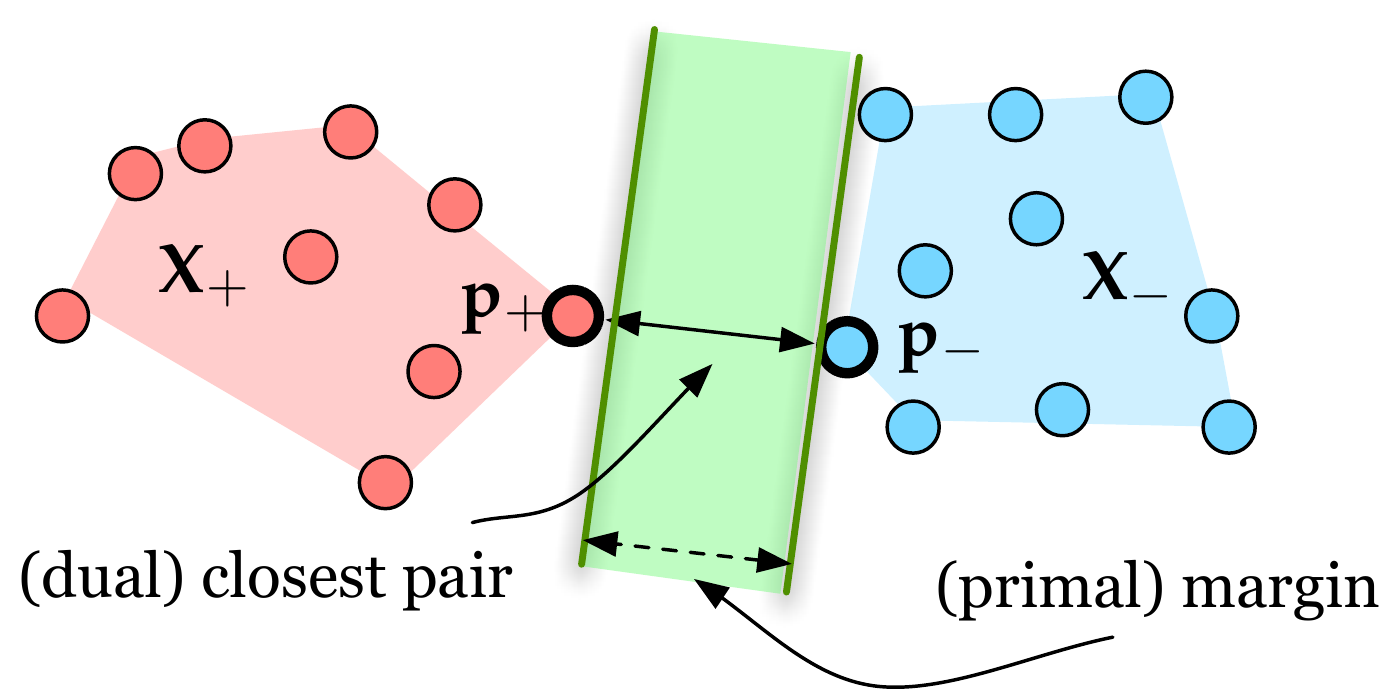}
  \caption{\footnotesize Illustration of primal-dual relationship for classification.}
  \label{fig:svm}
\end{figure}
From standard duality, the maximum margin SVM problem is equivalent to \emph{finding the shortest distance between the convex hulls of $\vec{X}_+$ and $\vec{X}_-$}. 
This shortest distance between the hulls will exist between two points on the respective hulls (see Figure~\ref{fig:svm}). 
Since these points are in the hulls, they can be expressed as some convex combination of the rows of $\vec{X}_+$ and $\vec{X}_-$, respectively. 
That is, if $\vec{p}_+$ is the closest point on the positive hull, then
$\vec{p}_+$ can be expressed as $\balpha_+^\top\vec{X}_+$, where $\balpha_+^\top\vec{1} = 1$ and
$\alpha_j \geq 0$, with a similar construction for $\vec{p}_-$ and $\balpha_-$. 

This in turn can be written as an optimization
\begin{align}
  \quad & \min_{\balpha} \frac{1}{2} \|\vec{p}_+-\vec{p}_-\|^2
  \label{eq:geom-svm}
\\\text{s.t.}\ 
  \quad & \balpha_+^\top \mathbf{1} = 1,
  \quad \balpha_-^\top \mathbf{1} = 1,
  \quad \balpha_+, \balpha_- \geq 0
  \notag
\end{align}
Collecting all the $\alpha$ terms together by defining $\alpha_j \triangleq
\alpha_{y_j,j}$, and expanding the distance term $\|\vec{p}_+-\vec{p}_-\|^2$, it is straightforward to show that
Problem~\eqref{eq:geom-svm} is equivalent to 
\begin{align}
  \quad & \min_{\balpha} \frac{1}{2} \balpha^{\top} \vec{Y} \vec{X} \vec{X}^\top \vec{Y} \balpha - \balpha^\top\vec{1}
  \label{eq:dual-svm} \\
  \text{s.t.}\ 
  \quad & \balpha^\top\vec{y} = 0,
  \quad \alpha_i \geq 0.
  \notag
\end{align}
where $\balpha^{\top} \vec{Y} \vec{X} \vec{X}^\top \vec{Y} \balpha$ is merely a
compact way of writing $\sum_{j,k\in
  X}\alpha_{j}\alpha_{k}y_jy_k\innerprod{\vec{x}_{j}}{\vec{x}_{k}}$.
Problem \eqref{eq:dual-svm} is of course the familiar dual SVM problem. 
The equivalence of \eqref{eq:geom-svm} and \eqref{eq:dual-svm} is well known, so we decline to prove it here; see \citet{Bennett:2000:DGS:645529.657972} for a proof of this equivalence.

\paragraph{Kernelizing the dual.}
The geometric interpretation of the dual does not change when the examples are transformed by a \emph{reproducing kernel Hilbert space} (RKHS).
The Euclidean norm of the base vector space in $\|\vec{p}_+-\vec{p}_-\|^2$ is merely substituted with the RKHS norm:
$$
  \|\vec{p}_+-\vec{p}_-\|_{\kappa}^2 = \kappa(\vec{p}_+,\vec{p}_+) + \kappa(\vec{p}_-,\vec{p}_-) - 2\kappa(\vec{p}_+,\vec{p}_-),
$$
where the \emph{kernel function} $\kappa$ stands in for the inner product.
This is dubbed the \emph{kernel distance} \cite{DBLP:journals/corr/abs-1103-1625} or the maximum mean
discrepancy \cite{mmd-start}.
The dual formulation then changes slightly, with the covariance
term $\vec{X} \vec{X}^\top$ being replaced by the kernel matrix $\vec{K}$. 
For brevity, we will define $\vec{G} \triangleq \vec{Y}\vec{K}\vec{Y}$.

\paragraph{Multiple kernel learning.}
\emph{Multiple kernel learning} is simply the SVM problem with the additional
complication that the kernel function is unknown, but is expressed as some
function of other known kernel functions.

Following standard practice~\cite{Lanckriet04MKLSDP} we assume that the kernel function is a \emph{convex combination} of other kernel functions; 
i.e., that there is some set of coefficients $\mu_i > 0$, that $\sum\mu_i = 1$, and that $\kappa = \sum\mu_i\kappa_i$ (which implies that the Gram matrix version is $\vec{K} = \sum\mu_i\vec{K}_i$). 
We regularize by setting $\trace(\vec{K}) = 1$~\cite{Lanckriet04MKLSDP}.
The dual problem then takes the following form~\citep{Lanckriet04MKLSDP}:
\begin{align} 
  \quad & \max_{\vec{K}} \min_{\balpha} 
  \quad \frac{1}{2}\balpha^{\top} \vec{G} \balpha - \balpha^{\top} \vec{1} \label{eq:mkl-main}
  \\
  \text{s.t.}\ 
  \quad & \vec{K} = \sum_{i=1}^m \mu_i \vec{K}_i,  
  \quad \trace(\vec{K}) = 1, 
  \quad \vec{K} \succeq 0, 
  \quad \bmu \geq 0  \notag
\end{align}
When juxtaposed with \eqref{eq:geom-svm} and \eqref{eq:dual-svm}, this can be interpreted as searching for the kernel that maximizes the shortest (kernel) distance between polytopes.

\section{Our Algorithm}
\label{sec:new-alg}

The MKL formulation of \eqref{eq:mkl-main} can be transformed (as we shall see later) into a
quadratically-constrained quadratic problem that can be solved by a number of
different solvers~\cite{Lanckriet04MKLSDP,mosek,sedumi}. However, this approach requires a memory
footprint of $\Theta(mn^2)$ to store all kernel matrices. Another approach
would be to exploit the $\min$-$\max$ structure of \eqref{eq:mkl-main} via an
alternating optimization: note that the problem of finding the shortest distance
between polytopes for a fixed kernel is merely the standard SVM
problem. There are two problems with this approach: \begin{inparaenum}[(a)] \item standard SVM
algorithms do not scale well with $m, n$, and \item it is not obvious how to adjust
kernel weights in each iteration.\end{inparaenum}

\paragraph*{Overview.}
\label{sec:overview}

Our solution exploits the fact that a QCQP is a special case of a general SDP. We do this in order to apply the \emph{combinatorial} primal-dual matrix multiplicative weight update (MMWU) algorithm of \citet{Arora:2007:CPA:1250790.1250823}. 
While the generic MMWU has expensive steps (a linear program and matrix exponentiation), we show how to exploit the structure of the MKL QCQP to yield a very simple alternating approach.
In the ``forward'' step, rather than solving an SVM, we
merely find two support vector that are ``most violating'' normal to the current
candidate hyperplane (in the lifted feature space). In the ``backward'' step, we
reweight the kernels involved using a matrix exponentiation that we reduce to a
\emph{closed form} computation without requiring expensive matrix
decompositions. 
Our speedup comes from the facts that 
\begin{inparaenum}
\item[(a)] the updates to support vectors are sparse (at most two in each step) and 
\item[(b)] that the backward step can be computed very efficiently. 
\end{inparaenum}
This allows us to reduce our memory footprint to $O(mn)$. 

\paragraph{QCQPs and SDPs.}
\label{sec:qcqps-sdps}

We start by using an observation due to \citet{Lanckriet04MKLSDP} to convert
\eqref{eq:mkl-main}\footnote{We note that \eqref{eq:qcqp-mkl} is the \emph{hard-margin} version of the MKL problem. The standard soft-margin variants can also be placed in this general framework~\cite{Lanckriet04MKLSDP}. 
For the $1$-norm soft margin, we add the  constraint that all terms of $\balpha$ are upper bounded by the margin constant $C$.
For the $2$-norm soft margin, another term $\frac{1}{C}\balpha^{\top}\balpha$ appears in the objective, or we can simply add a constant multiple of $\vec{I}$ to each $\vec{G}_i$.} 
 into the following QCQP:
\begin{align}
  \max_{\balpha,s} \quad & (2\balpha^{\top}\vec{1} - s) 
  \label{eq:qcqp-mkl} \\ 
  \textrm{s.t.}\ 
  \quad & s \geq \frac{1}{r_i} \balpha^{\top} \vec{G}_i \balpha, 
  \quad \balpha^{\top} \vec{y} = 0, 
  \quad \balpha \geq 0 \notag
\end{align}
where $\vec{G}_i = \vec{Y}\vec{K}_i\vec{Y}$, $r \in \b{R}^m$, and $r_i = \trace(\vec{K}_i)$.

Next, we rewrite \eqref{eq:qcqp-mkl} in canonical SDP form in order to apply the
MMWU framework:
\begin{align}
  \omega^* = \max_{\balpha, s} &\quad 2\balpha^{\top}\vec{1} - s 
  \label{eq:10} \\
  \text{s.t.} & \quad\forall i\in [1..m] \quad 
  \vec{Q}_i(\balpha) = \begin{pmatrix}
    \vec{I}_n & \vec{A}_i\balpha
    \\ (\vec{A}_i\balpha)^\top & s                
  \end{pmatrix}, 
  \notag \\ 
  & \quad \vec{Q}_i(\balpha) \succeq \vec{0}, 
  \quad \balpha^{\top}\vec{y} = 0, 
  \quad \balpha \geq \vec{0}. 
  \notag
\end{align}
where $\vec{A}_i^{\top} \vec{A}_i = \frac{1}{r_i}\vec{G}_i$ for all $i \in
[0..m]$. 

\paragraph{The MMWU framework.}
We give a brief overview of the MMWU framework of
\citet{Arora:2007:CPA:1250790.1250823} (for more details, the reader is directed
to Satyen Kale's thesis~\cite{kale2007efficient}). The approach starts with a ``guess''
$\omega$ for the optimal value $\omega^*$ of the SDP (and uses a binary search
to find this guess interleaved with runs of the algorithm). Assuming that this guess
at the optimal value is correct, the algorithm then attempts to find either a
feasible primal ($\vec{P}$) or dual assignment such that this guess is achieved. 
\begin{algorithm}[H]
  \footnotesize
  \caption{MMWU template~\cite{Arora:2007:CPA:1250790.1250823}\label{alg:arora}}
  \begin{algorithmic}
    \REQUIRE $\eps$, primal $\vec{P}^{(1)}$, rounds $T$, guess $\omega$
    \FOR{$t = 1 \ldots T$}
    \STATE \emph{forward:} Compute update to $\balpha^{(t)}$ based on constraints, $\vec{P}^{(t)}$ and $\balpha^{(t)}$
    \STATE \emph{backward:} Compute $\vec{M}^{(t)}$ from constraints and
    $\balpha^{(t)}$. 
     \STATE \qquad \qquad $\vec{W}^{(t+1)} \leftarrow e^{-\eps \sum_{t=1}^t \vec{M}^{(t)}}$
    \STATE \qquad \qquad $\vec{P}^{(t+1)} \leftarrow \frac{\vec{W}^{(t+1)}}{\text{Tr} (\vec{W}^{(t+1)})}$
    \ENDFOR
    \ENSURE $\vec{P}^{(T)}$ 
  \end{algorithmic}
\end{algorithm}
The process starts with some assignment to $\vec{P}^{(1)}$ (typically the identity matrix $\vec{I}$). 
If this assignment is both primal feasible and at most $\omega$, the process ends. 
Else, there must be some assignment to $\balpha$ (the dual) that ``witnesses'' this lack of feasibility or optimality, and it can be found by solving a linear program using the current primal/dual assignments and constraints (i.e., is positive, has dual value at least $\omega$, and satisfies constraints \eqref{eq:qcqp-mkl}). 

The primal constraints and $\balpha$ are then used to guide the search for a new primal assignment $\vec{P}^{(t+1)}$.
They are combined to form the matrix $\vec{Q}_i(\balpha^{(t)})$ (see \eqref{eq:qcqp-mkl}), and then adjusted to form an ``event matrix'' $\vec{M}^{(t)}$ (see Paragraph ``the backward step'' for details)\footnote{$\vec{M}^{(t)}$ generalizes the loss incurred by experts in traditional MWU -- by deriving $\vec{M}^{(t)}$ from the SDP constraints, the duality gap of the SDP takes the role of the loss.}.
Exponentiating the sum of all the observed $\vec{M}^{(t)}$ so far, the algorithm exponentially re-weights primal constraints that are more important, and the process repeats.
By minimizing the loss, the assignments to $\vec{P}^{(t)}$ and $\balpha^{(t)}$ are guaranteed to result in an SDP value that approximates $\omega^*$ within a factor of $(1+\epsilon)$.

\subsection{Our algorithm}
We now adapt the above framework to solve the MKL SDP given by \eqref{eq:10}. 
As we will explain below, we can assign $\omega^*$ \emph{a priori} in most cases and we can solve our problem with only one round of feasibility search.  
We denote the dual update in iteration $t$ by $\balpha^{(t)}$, the $i^{\text{th}}$ event matrix in iteration $t$ by $\vec{M}_i^{(t)}$ and the $i^{\text{th}}$ primal variable (matrix) in iteration $t$ by $\vec{P}_i^{(t)}$. $\vec{P}_i^{(t)}$ is closely related to the desired primal kernel coefficients $\mu_i$.  
We denote $\balpha = \sum_i \balpha^{(i)}$ as the accumulated dual assignment thus far and $\vec{M}_i = \sum_t
\vec{M}_i^{(t)}$ as the accumulated $i^{\text{th}}$ event matrix.

\subsubsection{The backward step}
\label{sec:backward-step}

It will be convenient to explain the backward step first. 
Given $\balpha^{(t)}$ and $\vec{Q}_i(\balpha^{(t)})$, we define 
$
\vec{M}_i^{(t)} \triangleq \frac{1}{2\rho}(\vec{Q}_i(\balpha^{(t)}) + \rho \vec{I}_{n+1})
$ 
where $\rho$ is a rate parameter to be set later. 
Note that $\vec{M}_i^{(t)}$ (and $\vec{M}^{(t)}$) is ``almost-diagonal'', taking the form 
$
  \begin{bmatrix} 
    a \vec{I}_n & \vec{u} 
    \\ \vec{u^\top} & a
  \end{bmatrix}
$.  
Such  matrices can be exponentiated in closed form.
\begin{lemma}
  The exponential of a matrix in the form
  $
  \begin{pmatrix}
    a\vec{I}_n & \vec{u} \\
    \vec{u}^{\top} & a
  \end{pmatrix},
  $
  where $a \ge 0$ and $\hat{\vec u} = \vec{u}/\|\vec u\|$, is
\[   e^a\Bigl[
    \begin{pmatrix}
      \cosh \|\vec{u}\|\hat{\vec u}\hat{\vec u}^{\top} & \sinh \|\vec{u}\| \hat{\vec u} \\
      \sinh \|\vec{u}\| \hat{\vec u}^{\top} & \cosh \|\vec{u}\| 
    \end{pmatrix} 
    +\begin{pmatrix}
      \vec{I}_n-\hat{\vec u}\hat{\vec u}^{\top} & 0 \\
      0 & 0
    \end{pmatrix}\Bigr].
\]
  \label{lemma:mat-exp}
\end{lemma}

\begin{proof}
We symbolically exponentiate an $n+1 \times n+1$ matrix of the form 
\[
\vec M = 
\begin{pmatrix}
  aI_n & \vec{u} \\
  \vec{u}^\top & a
\end{pmatrix}.
\]
Since this matrix is real and symmetric, its eigenvalues $\lambda_i$ are positive and its unit eigenvectors $\vec{v}_i$ form an orthonormal basis. The method that we use to symbolically exponentiate it is to express it in the form 
\[
\vec M = \sum_{i=1}^n \lambda_i\vec{v}_i\vec{v}_i^\top.
\]
The exponential then becomes 
\[
e^{\vec M} = \sum_{i=1}^n e^{\lambda_i}\vec{v}_i\vec{v}_i^\top.
\]
As a matter of notation, let $\hat{\vec u}$ be the unit vector such that $\|\vec u\|\hat{\vec u} = \vec{u}$.

\paragraph{Eigenvalues.}
\label{sec:eigenvalues}

The characteristic equation for $\vec{M}$ is not difficult to calculate. It is:
\begin{align}
  (\lambda - a)^{n-1}(\lambda^2 -2a\lambda + a^2-\|\vec u\|^2)
  = (\lambda - a)^{n-1}(\lambda - a + \|\vec u\|)(\lambda - a - \|\vec u\|).
  \label{eq:7}
\end{align}
This yields $n-1$ eigenvalues equal to $a$, and the other two equal to $a + \|\vec u\|$ and $a - \|\vec u\|$. We label them $\lambda_1$ and $\lambda_2$, respectively, and the rest are equal to $a$.

\paragraph{Eigenvectors.}
\label{sec:eigenvectors}

First we show that $\vec M$ has two eigenvectors of the form $(\vec u,\pm\|\vec{u}\|)^\top$:
\[
\begin{pmatrix}
  aI_n & \vec{u} \\
  \vec{u}^\top & a
\end{pmatrix}
\begin{pmatrix}
  \vec{u} \\ \pm\|\vec{u}\|
\end{pmatrix} = 
\begin{pmatrix}
  (a\pm\|\vec{u}\|)\vec{u} \\ 
  \|\vec u\|^2 \pm a\|\vec{u}\|
\end{pmatrix} = 
(a\pm\|\vec{u}\|)\begin{pmatrix}
  \vec{u} \\ 
  \pm\|\vec u\|
\end{pmatrix},
\]
so these are eigenvectors with eigenvalues $a\pm\|\vec{u}\|$. 
We will call the corresponding eigenvectors $\vec{v}_1$ and $\vec{v}_2$.
Since $\vec M$ is symmetric, all of its eigenvectors are orthogonal. 
The remaining eigenvectors are of the form $(\vec w,0)^\top$, where $\vec{w}^\top\vec{u} = 0$:
\[
\begin{pmatrix}
  aI_n & \vec{u} \\
  \vec{u}^\top & a
\end{pmatrix}
\begin{pmatrix}
  \vec{w} \\ 0
\end{pmatrix} = 
\begin{pmatrix}
  a\vec{w} \\ 
  0
\end{pmatrix}.
\]
Clearly the corresponding eigenvalue for any such eigenvector is $a$, so there are $n-1$ of them. The corresponding parts of these eigenvectors are labeled $\vec{w}_i$, where $3\leq i\leq n+1$, and we assume they are unit vectors.

\paragraph{The Exponential.}
\label{sec:exponential}
For unit eigenvectors $\vec{\hat v}_i$, since 
\[
e^{\vec M} = \sum_{i=1}^n e^{\lambda_i}\frac{\vec{v}_i\vec{v}_i^\top}{\|\vec{v}_i\|^2},
\]
and the eigenvalue $a$ is of multiplicity $n-1$, we have 
\begin{align*}
  e^{\vec M} 
  &= \frac{e^{\lambda_1}}{2\|\vec{u}\|^2} 
  \begin{pmatrix}
    \vec{u}\vec{u}^\top & \|\vec{u}\|\vec{u} \\
    \|\vec{u}\|\vec{u}^\top & \|\vec{u}\|^2
  \end{pmatrix}
  + \frac{e^{\lambda_2}}{2\|\vec{u}\|^2}
  \begin{pmatrix}
    \vec{u}\vec{u}^\top & -\|\vec{u}\|\vec{u} \\
    -\|\vec{u}\|\vec{u}^\top & \|\vec{u}\|^2
  \end{pmatrix}
  + e^a\sum_{i=3}^n 
  \begin{pmatrix}
    \vec{w}_i\vec{w}_i^\top & \vec{0} \\
    \vec{0}^\top & 0
  \end{pmatrix}
\\ &= e^a\left[
  \frac{e^{\|\vec{u}\|}}{2}
  \begin{pmatrix}
    \hat{\vec u}\hat{\vec u}^\top & \hat{\vec u} \\
    \hat{\vec u}^\top & 1
  \end{pmatrix}
  + \frac{e^{-\|\vec{u}\|} }{2}
  \begin{pmatrix}
    \hat{\vec u}\hat{\vec u}^\top & -\hat{\vec u} \\
    -\hat{\vec u}^\top & 1
  \end{pmatrix} 
  +
  \begin{pmatrix}
    I_{n} - \hat{\vec u}\hat{\vec u}^\top & \vec{0} \\
    \vec{0}^\top & 0
  \end{pmatrix}
  \right]
\\ &= e^a\left[
  \begin{pmatrix}
    \cosh\|\vec{u}\|\hat{\vec u}\hat{\vec u}^\top & \sinh\|\vec{u}\|\hat{\vec u} \\
    \sinh\|\vec{u}\|\hat{\vec u}^\top & \cosh\|\vec{u}\|
  \end{pmatrix}
  +
  \begin{pmatrix}
    I_{n} - \hat{\vec u}\hat{\vec u}^\top & \vec{0} \\
    \vec{0}^\top & 0
  \end{pmatrix}
  \right].
\end{align*}
The last term in the equality is due to the fact that $\vec{\hat u}$ and the $\vec{\hat w}_i$ form an orthonormal basis for $\mathbb{R}^n$, so $\vec{\hat u}\vec{\hat u}^\top + \sum\vec{\hat w}_i\vec{\hat w}_i^\top = I_n$.
\end{proof}

Lemma~\ref{lemma:mat-exp} implies that we can exponentiate the event matrix $\vec{M}^{(t)}$ (see Algorithm~\ref{alg:arora}) quickly, as promised. 
In particular, we set $\vec{P}_i^{(t+1)} = c \exp(-\varepsilon \sum_t\vec{M}_i^{(t+1)})$ where $c$ normalizes the matrix to have unit trace.

\paragraph{Practical considerations.}
In Lemma~\ref{lemma:mat-exp}, large inputs to the functions $\exp$, $\cosh$, and $\sinh$ will cause them to rapidly overflow even at double-precision range. 
Fortunately there are two steps we can take. 
First, $\cosh(x)$ and $\sinh(x)$ converge exponentially to $\exp(x)/2$, so above a high enough value, we can simply approximate $\sinh(x)$ and $\cosh(x)$ with $\exp(x)/2$. 

Because $\exp$ can overflow just as much as $\sinh$ or $\cosh$, this doesn't solve the problem completely.  
However, since $\vec P$ is always normalized so that $\trace(\vec P)=1$, we can multiply the elements of $\vec P$ by any factor we choose and the factor will be normalized out in the end.
So above a certain value, we can use $\exp$ alone and throw a ``quashing'' factor ($e^{-\phi-q}$) into the equations before computing the result, and it will be normalized out later in the computation (this also means that we can ignore the $e^a$ factor). 
For our purposes, setting $q=20$ suffices. 
This trades overflow for underflow, but underflow can be interpreted merely as one kernel disappearing from significance.

Note that the structure of $\vec{P}^{(t)}$ also allows us to avoid storing it explicitly, since $(a\vec{I}) \bullet (b\hat{\vec u}\hat{\vec u}^{\top}) = ab$.
We need only store the coefficients of the blocks of the $\vec{P}_i^{(t)}$.

\paragraph{The exponentiation algorithm.}
From $\vec{M}_{i}^{(t)}$ in Algorithm \ref{alg:arora} and \eqref{eq:10}, we have $\vec{M}_{i}^{(t)} = \frac{1}{2\rho}(\vec{Q}_i(\balpha^{(t)}) + \rho \vec{I}_{n+1})$, where $\rho$ is a program parameter which is explained in \ref{ssec:algo}. 

Our 
$
\vec{Q}_i(\balpha) = \begin{pmatrix}
  \vec{I}_n & \vec{A}_i\balpha  \\
  (\vec{A}_i\balpha)^\top & 1
\end{pmatrix}
$ 
is of the form													
$\begin{pmatrix}
  a\vec{I}_n & \vec{u}_i \\
  \vec{u}_i^{\top} & a
\end{pmatrix}$, 
where $a=1$ $\forall i$ and $\vec{u}_i = \vec{A}_i \balpha$.
So we have
\begin{align}
  \label{eq:ui}
  \vec{u}^{\top}_i \vec{u}_i 
  = (\vec{A}_i \balpha)^{\top} \vec{A}_i \balpha 
  = \balpha^{\top} \vec{A}^{\top}_i \vec{A}_i \balpha 
  = \balpha^{\top} \frac{1}{r_i} \vec{G}_i \balpha
\end{align}
where the last equality follows from $\vec{A}_i^{\top} \vec{A}_i = \frac{1}{r_i}\vec{G}_i$ 
(cf. \eqref{eq:10}).
As we shall show in Algorithm~\ref{alg:mwu-qcqp}, at each iteration the matrix to be exponentiated is a sum of matrices of the form $\frac{1}{2\rho}(\vec{Q}_i(\sum_{t=1}^\tau \balpha^{(t)}) + \rho t \vec{I}_{n+1})$, so Lemma~\ref{lemma:mat-exp} can be applied at every iteration. 

We provide in detail the algorithm we use to exponentiate the matrix $\vec{M}$. This subroutine is called from Algorithm~\ref{alg:mwu-qcqp} in Section~\ref{sec:new-alg}. 

\begin{algorithm}[!h]
  \footnotesize
  \caption{{\sc Exponentiate-}$M$}
  \begin{algorithmic}
    \REQUIRE $\vec{y}$, $\balpha$, $\{\vec{G}_i\}$, $\eps'$, $\rho$
    \FOR {$i \in [1..m]$}
    \STATE $\|\vec{u}_i\| \leftarrow \sqrt{\balpha^{\top}\vec{G}_i\balpha}$
    \STATE $\vec{g}_i \leftarrow \frac{1}{\|\vec{u}_i\|}\vec{G}_i\balpha$
    \STATE $\|\vec{u}_i\| \leftarrow \frac{\eps'}{2\rho}\|\vec{u}_i\|$
    \ENDFOR
    \STATE $q \leftarrow \max_i \|\vec{u}_i\|$
    \IF {$q < 20$}
    \FOR {$i \in [1..m]$}
    \STATE $p_i^{11} \leftarrow \cosh(\|\vec{u}_i\|)$
    \STATE $p_i^{12} \leftarrow -\sinh(\|\vec{u}_i\|)$
    \ENDFOR
    \STATE $e_{\vec{M}} \leftarrow 1$
    \ELSE
    \FOR {$i \in [1..m]$}
    \STATE $p_i^{11} \leftarrow e^{\|\vec{u}_i\|-q}$
    \STATE $p_i^{12} \leftarrow -p_i^{11}$
    \ENDFOR
    \STATE $e_{\vec{M}} \leftarrow e^{-q}$
    \ENDIF
    \STATE $S \leftarrow m(n-1)e_{\vec{M}} + 2\sum_i p_i^{11}$
    \FOR {$i \in [1..m]$}

    \STATE $p_i^{12} \leftarrow p_i^{12} / S$
    \ENDFOR
    \STATE $\vec{g} \leftarrow \sum_i2p_i^{12}\vec{g}_i$
    \STATE Return $\vec{p}^{12}$, $\vec{g}$
  \end{algorithmic}
  \label{alg:exponentiate-m}
\end{algorithm}

\subsubsection{The forward step}
\label{sec:forward}

In the forward step, we wish to check if our primal solution $\vec{P}$ is feasible and
optimal, and if not find updates to $\balpha^{(t)}$. 
In order to do so, we apply the MMWU template.
The goal now is to find $\balpha^{(t)}$ such that 
\[
\sum_i \vec{Q}_i(\balpha^{(t)}) \bullet \vec{P}_i \ge 0,\ 
\balpha^{(t)} \ge 0,\ 
(\balpha^{(t)})^{\top}\vec{y} = 0,\ \text{and}\ 
(\balpha^{(t)})^{\top}\vone = 1.
\]
The existence of such a $\balpha^{(t)}$ will prove that the current guess $\vec{P}^{(t)}$ is either primal infeasible or suboptimal (see \citet{Arora:2007:CPA:1250790.1250823} for details). 

We now exploit the structure of $\vec{P}^{(t)}$ given by Lemma \ref{lemma:mat-exp}. 
In particular, let $p_i^{11} = p_i^{22} = e^a \cosh \|\vec{u}_i\|/\trace{\vec{P}}$ and $p_i^{12} = -e^a \sinh \|\vec{u}_i\|/\trace{\vec{P}}$. 
So 
\begin{equation*}
  \vec{Q}_i(\balpha^{(t)}) \bullet \vec{P}_i 
  = \begin{pmatrix}
    0 & \vec{A}_i\balpha^{(t)}  \\
    (\vec{A}_i\balpha^{(t)})^\top & 0
  \end{pmatrix} \bullet \vec{P}_i
  + \vec{I}_{n+1}  \bullet \vec{P}_i
  = 2p_i^{12}\hat{\vec{u}}_i^{\top}\vec{A}_i\balpha^{(t)} + \trace(\vec{P}_i)
\end{equation*}
$\sum_i \vec{Q}_i(\balpha^{(t)}) \bullet \vec{P}_i \geq 0$ then reduces to:
\begin{equation}
  (\balpha^{(t)})^{\top} \sum_{i=0}^m (2p_i^{12}\vec{A}_i\hat{\vec u}_i) \geq - \trace(\vec{P}). \label{eq:4}
\end{equation}
The right hand side is the negative trace of $\vec{P}$ (which is normalized to $1$), so this becomes
\begin{equation}
  \label{eq:oracle-constraint}
  (\balpha^{(t)})^{\top}\sum_i2p_i^{12}\vec{g}_i \geq -1,
\end{equation}
where $\vec{g}_i = (\frac{1}{r_i}\vec{G}_{i}\balpha)/(\frac{1}{r_i}\balpha^{\top}\vec{G}_{i}\balpha)^{1/2}$.
If we let $\vec{g} = \sum_i2p_i^{12}\vec{g}_i$ (which can be calculated at the end of the \emph{backward} step), then we have simply
$\vec{g}^\top\balpha \geq -1$ which is a simple collection of linear constraints that can \emph{always} be satisfied\footnote{The current margin borders a convex combination of points from each side. If we could not find a point such that the inequality is satisfied, then no point from the convex combination can be found on or past the margin, which is impossible.}.

Geometrically, $\vec{g}$ gives us a way to examine the training points that are farthest away from the margin. 
The higher a value $g_j$ is, the more it violates the current decision boundary. 
In order to find a $\balpha$ that satisfies \eqref{eq:oracle-constraint}, we simply choose the highest elements of $\vec{g}$ that correspond to both positive and negative labels, then set each corresponding entry in $\balpha$ to $1/2$. 
Algorithm \ref{alg:oracle-mkl} describes the pseudo-code for this process.

\begin{algorithm}[H]
  \footnotesize
  \caption{\textsc{Find}-$\balpha$}
  \label{alg:oracle-mkl}
  \begin{algorithmic}
    \REQUIRE $\vec{y}$, $\vec{g}$
    \STATE $P \leftarrow \{i \mid \vec{y}_i = 1\}$, $N \leftarrow \{i \mid \vec{y}_i = -1\}$
    \STATE $i_P \leftarrow \argmax_{i\in P} \vec{g}_i$, $i_N \leftarrow \argmax_{i\in N} \vec{g}_i$
    \STATE $\balpha \leftarrow \vec{0}$
    \STATE $\balpha_{i_P} \leftarrow 1/2$, $\balpha_{i_N} \leftarrow 1/2$
    \RETURN $\balpha$
    \ENSURE $\balpha$ s.t. $\balpha \geq 0$, $\balpha^{\top}\vec{1} = 1$, $\balpha^{\top}\vec{y} = 0$
  \end{algorithmic}
\end{algorithm}

\paragraph{Practical Considerations.} 
We highlight two important practical consequences of our formulation.
First, the procedure produces a very sparse update to $\balpha$: in each iteration, only two coordinates of $\balpha$ are updated. 
This makes each iteration very efficient, taking only linear time. 
Second, by expressing $\vec{u}_i$ in terms of $\vec{g}_i$ we never need to explicitly compute $\vec{A}_{i}$ (as $\vec{u}_i = \vec{A}_i\balpha$), which in turn means that we do not need to compute the (expensive) square root of $\vec{G_i}$ explicitly. 

Another beneficial feature of the dual-finding procedure for MKL is that terms involving the primal variables $\vec{P}$ are either normalized (when we set the trace of $\vec{P}$ to $1$) or eliminated (due to the fact that we have a compact closed-form expression for $\vec{P}$), \emph{which means that we never have to explicitly maintain $\vec{P}$}, save for a small number ($4m$) of variables. 

\subsection{Avoiding binary search for $\omega$}
\label{sec:avoid-binary-search}

The objective function in \eqref{eq:10} is \emph{linear}, so we can scale $s$ and $\balpha$ and use the fact that $s = \balpha^{\top}\vec{1} = \omega$ to transform the problem\footnote{This fact follows from the KKT conditions for the original problem. 
The support constraints of the SVM problem can be written as $\vec{G}\balpha + b\vec{y} \geq \vec{1}$. 
If we multiply both sides of this inequality by $\balpha^{\top}$ then it becomes an equality (by complementary slackness): 
$\balpha^{\top}\vec{G}\balpha = \balpha^{\top}\vec{1}$. 
$s$ is a substitution for $\balpha^{\top}\vec{G}\balpha$ in the MKL problem~\cite{Lanckriet04MKLSDP} so $s = \balpha^{\top}\vec{1} = \omega$ as well.}:
\begin{align*}
&  \textrm{find}\ \quad  \balpha \quad  \textrm{s.t.} \\
 &   1/\omega \geq \frac{1}{r_i} \balpha^{\top} \vec{G}_i \balpha, 
  \quad    \balpha^{\top} \vec{y} = 0, \quad \balpha^{\top}\vec{1} = 1, \quad \balpha \geq 0,
\end{align*}
where $\balpha = \omega\balpha$. 
The first constraint can be transformed back into an optimization; 
that is, $\min_\omega \max_{\balpha,i} \frac{1}{r_i} \balpha^{\top} \vec{G}_i \balpha$, subject to the remaining linear constraints. 
Because $\omega$ does not figure into the maximization, we can compute $\omega$ simply by maximizing $\frac{1}{r_i} \balpha^{\top} \vec{G}_i \balpha$.  
Practically, this means that we simply add the constraint $\balpha^{\top}\vec{1} = 1$, and the ``guess'' for $\omega$ is set to $1$. 
We then \emph{know} the objective, and only one iteration is needed, so the binary search is eliminated.

\subsection{Extracting the solution from the MMWU}
\label{ssec:soln}
We start by observing that $\sum_{i=1}^m\vec{Q}_i\bullet \vec{P}_i = 0$ (by complementary slackness), 
which can rewritten as 
\begin{align}
\label{eq:final1}
& \sum_{i=1}^m \frac{2p_i^{12}}{r_i} \left( \frac{r_i}{\balpha^\top\vec{G}_i\balpha} \right)^{1/2} \balpha^\top\vec{G}_i\balpha = 1.
\end{align}

Now recall (from section \ref{sec:back}) that 
$
\balpha^\top\vec{G}\balpha 
= \sum_{i=1}^m \mu_i\cdot\balpha^\top\vec{G}_i\balpha, 
$
and we also use the fact that $\balpha^\top\vec{G}\balpha = \balpha^\top\vec{1} = \omega = 1$. 
Combining the above two we have:
\begin{align}
\label{eq:final2}
\sum_{i=1}^m \mu_i \cdot \balpha^\top\vec{G}_i\balpha = 1
\end{align}

Matching \eqref{eq:final1} with \eqref{eq:final2} suggests that $\frac{2p_i^{12}}{r_i} \left( \frac{r_i}{\balpha^\top\vec{G}_i\balpha} \right)^{1/2}$ is the appropriate choice for $\mu_i$. 

\subsection{Putting it all together}
\label{ssec:algo}
Algorithm~\ref{alg:mwu-qcqp} summarizes the discussion in this section. 
The parameter $\eps$ is the error in approximating the objective function, but its connection to classification accuracy is loose.
We set the actual value of $\eps$ via cross-validation (see Section~\ref{sec:results}). 
The parameter $\rho$ is the \emph{width} of the SDP, a parameter that indicates how much the solution can vary at each step. 
$\rho$ is equal to the maximum absolute value of the eigenvalues of $\vec{Q}_i(\balpha^{(t)})$, for any $i$~\cite{Arora:2007:CPA:1250790.1250823}. 

\begin{lemma}
  $\rho$ is bounded by $3/2$.
\end{lemma} 
\begin{proof}
  $\rho$ is defined as the maximum of $\|\vec{Q}(\balpha^{(t)})\|$ for all $t$. 
  Here $\|\cdot\|$ denotes the largest eigenvalue in absolute value~\cite{Arora:2007:CPA:1250790.1250823}. 
  Because $s = \omega = 1$ (see Section~\ref{sec:new-alg}), the eigenvalues of $\vec{Q}_i(\balpha^{(t)})$ are $1$ (with multiplicity $n-1$), and $1\pm\|\vec{A}_i\balpha^{(t)}\|$. 
  The greater of these in absolute value is clearly $1+\|\vec{A}_i\balpha^{(t)}\|$.

  $\|\vec{A}_i\balpha^{(t)}\|$ is equal to 
  \[
  ((\balpha^{(t)})^T\vec{A}_i^T\vec{A}_i\balpha^{(t)})^{\frac{1}{2}}
  = \left( \frac{1}{r_i}(\balpha^{(t)})^T\vec{G}_i\balpha^{(t)} \right)^{\frac{1}{2}}.
  \]
  $\balpha^{(t)}$ always has two nonzero elements, and they are equal to $1/2$. They also correspond to values of $\vec{y}$ with opposite signs, so if $j$ and $k$ are the coordinates in question, $(\balpha^{(t)})^T\vec{G}_i\balpha^{(t)} \leq (1/4)(\vec{G}_{i(jj)} + \vec{G}_{i(kk)})$, because $\vec{G}_{i(jk)}$ and $\vec{G}_{i(kj)}$ are both negative. 
Because of the factor of $1/r_i$, and because $r_i$ is the trace of $\vec{G}_i$, $\|\vec{A}_i\balpha^{(t)}\| \leq 1/2$. 
This is true for any of the $i$, so the maximum eigenvalue of $\vec{Q}(\balpha^{(t)})$ in absolute value is bounded by $1 + 1/2 = 3/2$.
\end{proof}

\paragraph{Running time.} 
Every iteration of Algorithm~\ref{alg:mwu-qcqp} will require a call to \textsc{Find}-$\balpha$, a call to \textsc{Exponentiate}-$M$ and an update to $\vec{G}_i\balpha$ and $\balpha^{\top}\vec{G}_i\balpha$. 
\textsc{Find}-$\balpha$ requires a linear search for two maxima in $\vec{g}$, so the first is $O(n)$. 
The latter are each $O(mn)$, which dominate \textsc{Find}-$\balpha$.

Algorithm~\ref{alg:mwu-qcqp} requires a total of $T$ iterations at most, where $T = \frac{8\rho^2}{\eps^2}\ln(n)$.
Since we only require one run of the main algorithm, the running time is bounded by 
$
  O\left( mn \ln(n) \frac{1}{\eps^2} \right).
$

\begin{algorithm}
  \footnotesize
  \caption{\mwumkl}
  \label{alg:mwu-qcqp}\
  \begin{algorithmic}
    \REQUIRE $\vec{g}^{(1)} = \vec{0}$;
    \\$\rho$, the width; 
    \\$\eps$, the desired approximation error 
    \STATE Set $\eps' = -\ln(1-\frac{\eps}{2\rho})$
    \STATE Set $T = \frac{8\rho^2}{\eps^2}\ln(n)$
    \REPEAT[$T$ times] 
    \STATE Get $\balpha^{(t)}$ from Algorithm~\ref{alg:oracle-mkl}
    \IF{Algorithm~\ref{alg:oracle-mkl} failed}
    \STATE Return
    \ENDIF
    \STATE Update $\balpha = \balpha + \balpha^{(t)}$
    \STATE Set $\vec{M}_i^{(t)} = \frac{1}{2\rho} \left( \vec{Q}_i(\balpha^{(t)}) + \rho \vec{I}_{n+1} \right)$
    \STATE Set $\vec{W}_i^{(t)} = e^{-\eps' \sum_{t=1}^T \vec{M}_i^{(t)}}$
    \STATE Set $\vec{P}_i^{(t+1)} = \vec{W}_i^{(t)}/\trace(\vec{W}_i^{(t)})$
    \STATE Compute $\vec{g}^{(t+1)}$ from $\vec{P}^{(t+1)}$, $\{\vec{G}_i\}$, and $\balpha$
    \UNTIL{$t = T$}
    \STATE Return $\frac{1}{T}\balpha$, $\vec{P}^{(T+1)}$
  \end{algorithmic}
\end{algorithm}

\section{Experiments}
\label{sec:results}

In this section we compare the empirical performance of \mwumkl with other multiple kernel learning algorithms. 
Our results have two components: 
(a) \emph{qualitative} results that compares test accuracies on small scale datasets, and 
(b) \emph{scalability} results that compares training time on larger datasets.

We compare \mwumkl with the following baselines: 
(a) \uni (uniformly weighted combination of kernels), and
(b) LibLinear~\cite{REF08a} with Nystr\"om kernel approximations for each kernel (hereafter referred to as \llplus).
We evaluate these MKL methods on binary datasets from UCI data repository. 
They include: 
(a) small datasets \iono, \cancr, \pima, \sonar, \heart, \vote, \wdbc, \wpbc, 
(b) medium dataset \mush, and 
(c) comparatively larger datasets \adult, \cod, and \web. 

Classification accuracy and kernel scalability results are presented on small and medium datasets (with many kernels). 
Scalability results (with $12$ kernels due to memory constraints) are provided for large datasets. 
Finally, we show results for lots of kernels on small data subsets.

\paragraph{Uniform kernel weights.}
\uni is simply LibSVM~\cite{libsvm} run with a kernel weighted equally amongst all of the input kernels (where the kernel weights are normalized by the trace of their respective Gram matrices first).
The performance of \uni is on par or better than \llplus on many datasets (see Figure~\ref{fig:mix-small-error}) and the time is similar to \mwumkl.
However \uni does not scale well due to the poor scaling of LibSVM beyond a few thousand samples (see Figure~\ref{fig:cod-time-scale}), because of the need to hold the entire Gram matrix in memory
\footnote{This is true even when LibSVM is told to use one kernel, which it can compute on the fly -- the scaling of LibSVM is $O(n^2)$ - $O(n^3)$~\cite{libsvm}, poor compared to \mwumkl and \llplus with increasing sample size.}.
We employ Scikit-learn~\cite{scikit-learn} because it offers efficient access to LibSVM.

\paragraph{LibLinear~\cite{REF08a} with Nystr\"om kernel approximations~\cite{williams2001using,yangnystrom} (\llplus).}
One important observation about multiple kernel learning is that \uni performs as well or better than many MKL algorithms with better efficiency.
Along this same line of thought, we should consider comparison against methods that are as simple as possible.
One of the very simplest algorithms to consider is to use a linear classifier (in this case, LibLinear~\cite{REF08a}), and transform the features of the data with a kernel approximation.
For our purposes, we use Nystr\"om approximations as described by \citet{williams2001using} and discussed further by \citet{yangnystrom}. 
Because LibLinear is a primal method, we don't need to scale each kernel -- each kernel manifests as a set of features, which the algorithm weights by definition.

For the Nystr\"om feature transformations, one only needs to specify the kernel function and the number of sample points desired from the data set. 
We usually use $150$ points, unless memory constraints force us to use fewer.
Theoretically, if $s$ is the number of sample points, $n$ the number of data points, and $m$ the number of kernels, then we would need space to store $O(snm)$ double-precision floats.
With regard to time, the training task is very rapid -- the transformation is the bottleneck (requiring $O(s^2mn)$ time to transform every point with every kernel approximation). 

We employ Scikit-learn~\cite{scikit-learn} for implementations of both the linear classifier and the kernel approximation because 
(a) this package offloads linear support-vector classification to the natively-coded LibLinear implementation,
(b) it offers a fast kernel transformation using the NumPy package, and
(c) Scikit-learn makes it very easy and efficient to chain these two implementations together.
In practice this method is very good and very fast for low numbers of kernels (see Figures~\ref{fig:mix-small-error}, \ref{fig:adult-error}, and \ref{fig:web-error}).
For high numbers of kernels, this scaling breaks down due to time and memory constraints (see Figure~\ref{fig:mix-time-ker-size}).

\paragraph{Legacy MKL implementations.}
In all cases, we omit the results for older MKL algorithm implementations such as 
(a) \silp~\cite{Sonnenburg06MKLSILP},
(b) \sdp~\cite{Lanckriet04MKLSDP}, 
(c) \smkl~\cite{Rakotomamonjy07MKL}, 
(d) \lmkl~\cite{DBLP:conf/nips/XuJKL08}, and 
(e) \glmkl~\cite{DBLP:conf/icml/XuJYKL10}
which take significantly longer to complete, have no significant gain in accuracy, and do not scale to any datasets larger than a few thousand samples. 
For example, on \sonar (one of the smallest sets in our pool), each iteration of \silp takes about $4500$ seconds on average whereas \uni requires $0.03$ seconds on average. 

\paragraph{Experimental parameters.}
\begin{table}[!htbp]
  \footnotesize
  \centering
  \begin{tabular}{|c|c|c|c|}
    \hline
		\textbf{Size} & \textbf{Dataset}       & \textbf{\#Points}  & \textbf{\#Dim}\\
    \hline
                  & \cancr     & 683   & 9 
        \\        & \heart   & 270   & 13
				\\        & \iono    & 351   & 33
        \\ Small  & \pima    & 768   & 8
        \\        & \sonar   & 208   & 60
        \\        & \vote    & 435   & 16
        \\        & \wdbc    & 569   & 30
        \\        & \wpbc    & 198   & 33
        \\
    \hline
        Medium    & \mush      & 8124  & 112
		\\
    \hline
                  & \adult     & 39073 & 123
		    \\ Large  & \cod     & 47628 & 8
        \\        & \web     & 64700 & 300
        \\
    \hline
  \end{tabular}   
  \caption{\footnotesize Datasets used in experiments.}
  \label{tab:datasets} 
\end{table}

Similar to \citet{Rakotomamonjy07MKL} and \citet{DBLP:conf/icml/XuJYKL10}, we test our algorithms on a base kernel family of $3$ polynomial kernels (of degree $1$ to $3$) and $9$ Gaussian kernels.
Contrary to~\cite{Rakotomamonjy07MKL,DBLP:conf/icml/XuJYKL10}, however, we test with Gaussian kernels that have a tighter range of bandwidths ($\{2^{0}, 2^{1/2}, \ldots, 2^{4}\}$, instead of $\{2^{-3}, 2^{-2}, \ldots, 2^{5}\}$). 
The reason for this last choice is that our method actively seeks solutions for each of the kernels, and kernels that encourage overfitting the training set (such as low-bandwidth Gaussian kernels) pull \mwumkl away from a robust solution.

For small datasets, kernels are constructed using each single feature and are repeated $30$ times with different train/test partitions.
For medium and large datasets, due to memory constraints on \llplus, we test only on $12$ kernels constructed using all features, and repeat only $5$ times. 
All kernels are normalized to trace $1$. 
Results from small datasets are presented with a $95$\% confidence interval that the median lies in the range.
Results from medium-large datasets present the median, with the min and max values as a range around the median.
In each iteration, $80$\% of the examples are randomly selected as the training data and the remaining $20$\% are used as test data. 
Feature values of all datasets have been scaled to $[0,1]$. 
SVM regularization parameter $C$
is chosen by cross-validation.
For example, in Figure~\ref{fig:mix-small-error} results are presented for the best value of $C$ for each dataset and algorithm. 

For \mwumkl, we choose $\eps$ by cross-validation. 
Most datasets get $\eps = 0.2$, but the exceptions are \web ($\eps = 0.07$),  \cod ($\eps = 0.07$), and \adult ($\eps = 0.05$).
Contrary to existing works we do not compare the number of SVM calls (as \mwumkl does not explicitly use an underlying SVM) and the number of kernels selected.

Experiments were performed on a machine with an Intel\textsuperscript{\textregistered} Core\textsuperscript{TM} 2 Quad CPU ($2.40$ GHz) and 2GB RAM. 
All methods have an outer test harness written in Python. 
\mwumkl also uses a test harness in Python with an inner core written in C++.

\begin{figure}[!htbp]
  \centering
  \includegraphics[width=0.8\columnwidth]{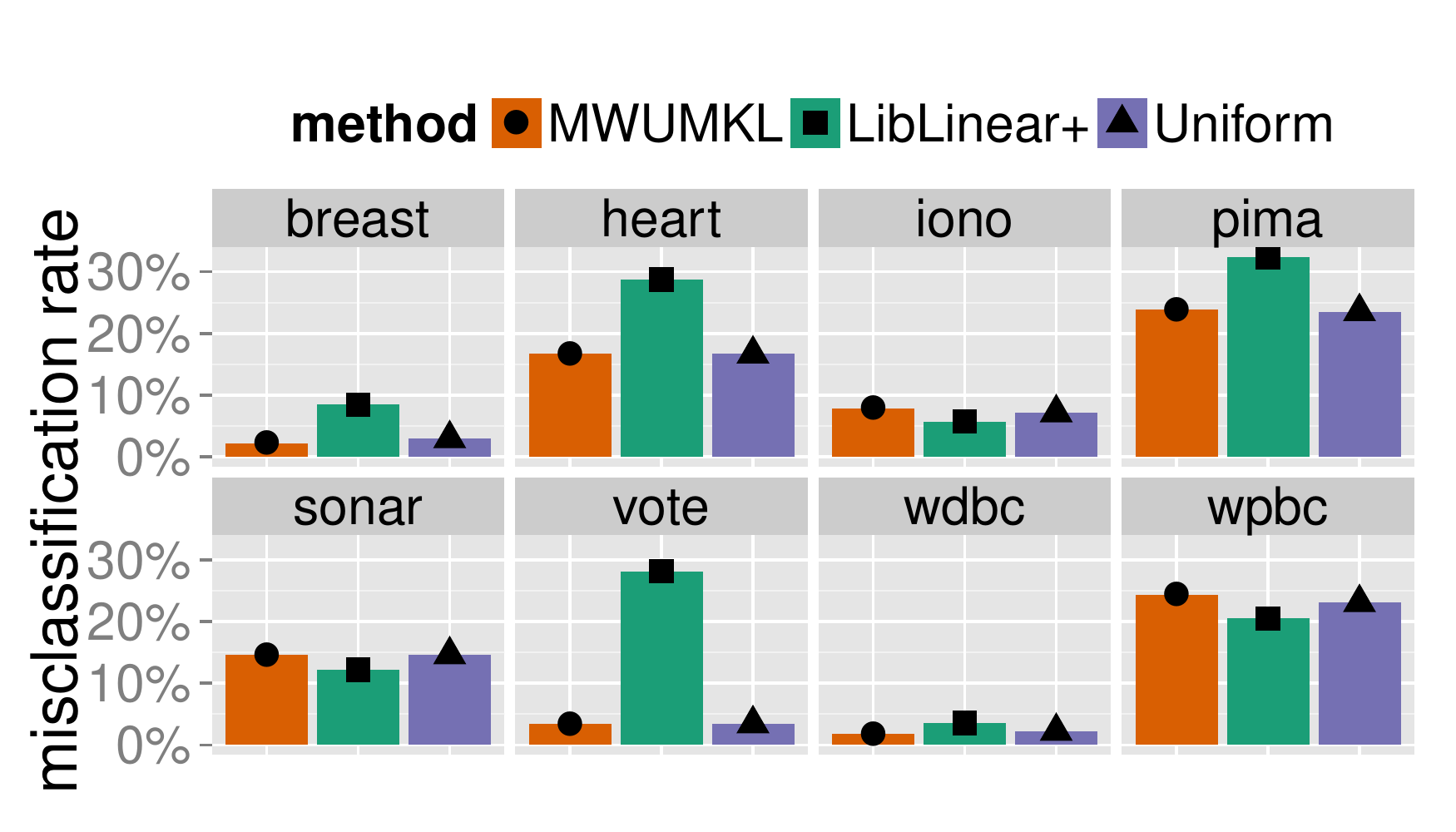}
  \caption{\footnotesize Median misclassification rate for small datasets.}
  \label{fig:mix-small-error}
\end{figure}

\paragraph{Accuracy.}
On small datasets our goal is to show that \mwumkl compares favorably with \llplus and \uni in terms of test accuracies. 

In Figure~\ref{fig:mix-small-error} we present the median misclassification rate for each small dataset over 30 random training/test partitions.
In each case, we train the classifier with $12$ kernels for each feature in the dataset, and each kernel only operates on one feature.
We are able either to beat the other methods or remain competitive with them.

\begin{figure}[!htbp]
  \centering
  \includegraphics[width=0.8\columnwidth]{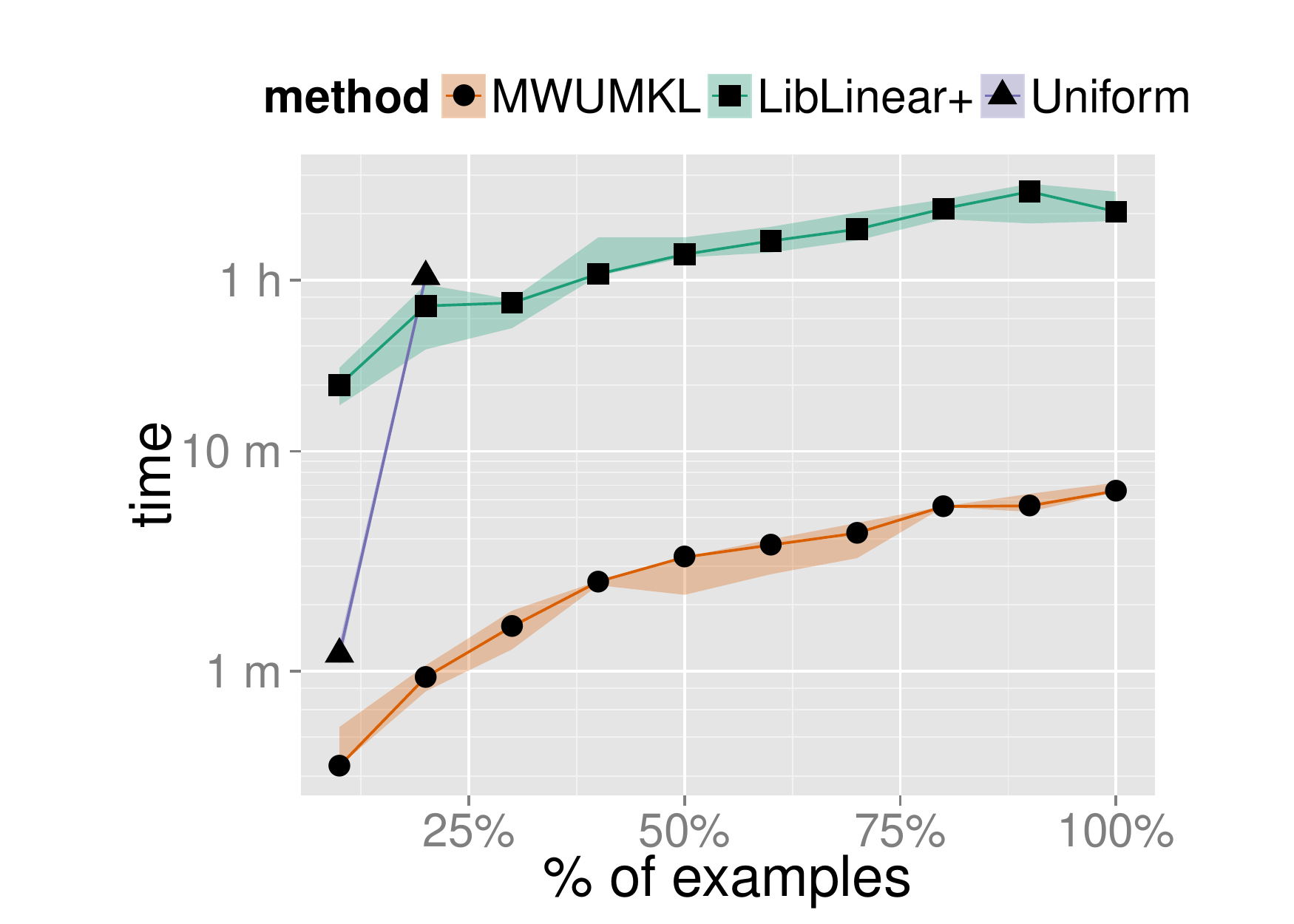}
  \caption{\footnotesize \cod ($n=59535$, $d=8$) with $12$ kernels.}
  \label{fig:cod-time-scale}
\end{figure}

\paragraph{Data Scalability.}
Both \mwumkl and \llplus are much faster as compared with \uni. 
At this point, \adult, \cod, and \web are large enough datasets that \uni fails to complete because of memory constraints.
This can be seen in Figure~\ref{fig:cod-time-scale}, where we plot training time versus the proportion of the training data used -- the training time taken by \uni rises sharply and we are unable to train on this dataset past $11907$ points.
Hence, for the remaining experiments on large datasets, we compare \mwumkl with \llplus.
In Figures~\ref{fig:adult-error} and \ref{fig:web-error}, we choose a random partition of train and test, and then train with increasing proportions of the training partition
(but always test with the whole test partition).
With more data, our algorithm settles in to be competitive with \llplus.
\begin{figure}[!htbp]
  \captionsetup[subfigure]{labelformat=empty}
  \centering
	\subfloat[]{
		\includegraphics[width=0.8\columnwidth]{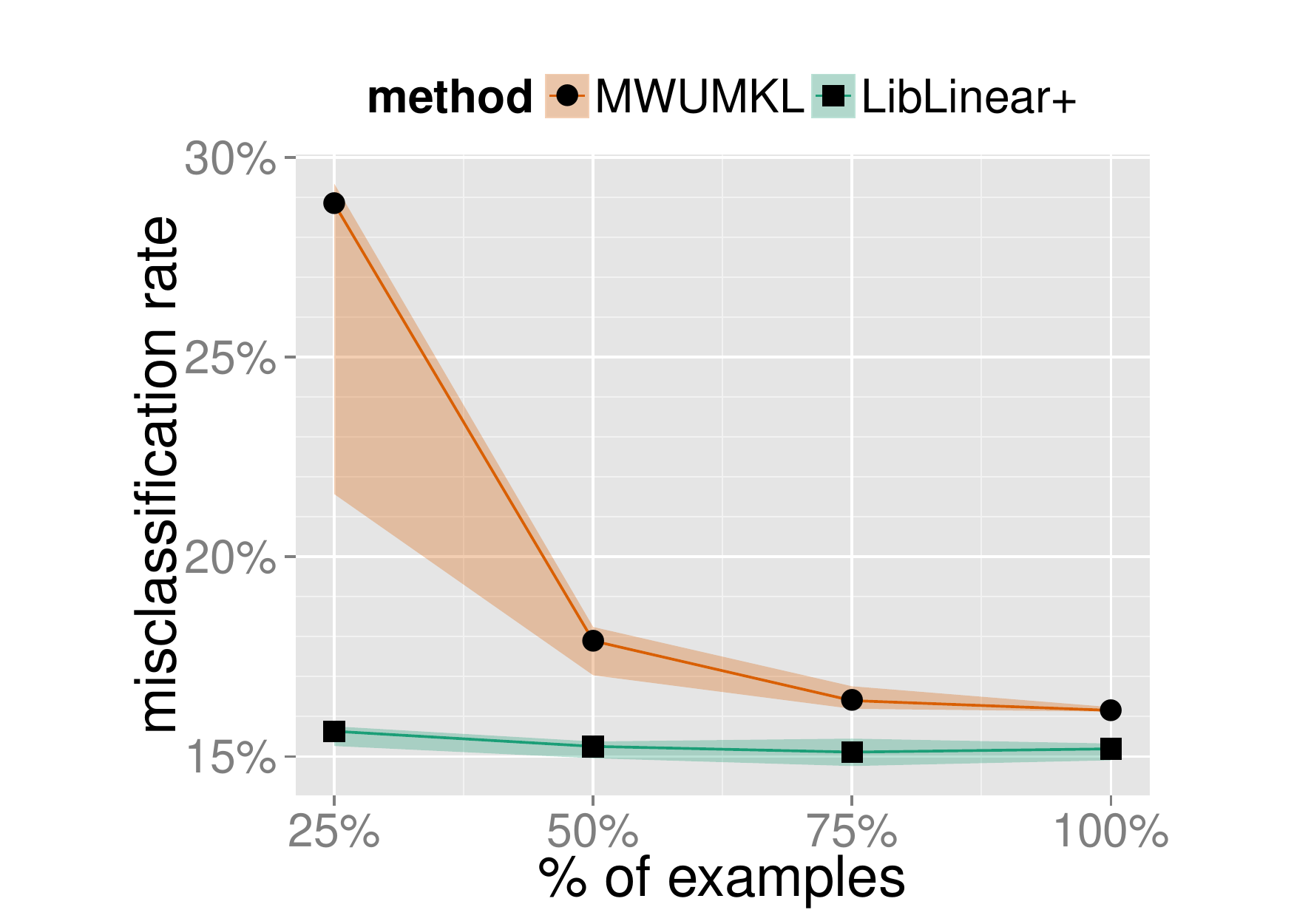}
		\label{fig:adult-error}
	}\\
	\vspace{-0.25in}
	\subfloat[]{
		\includegraphics[width=0.8\columnwidth]{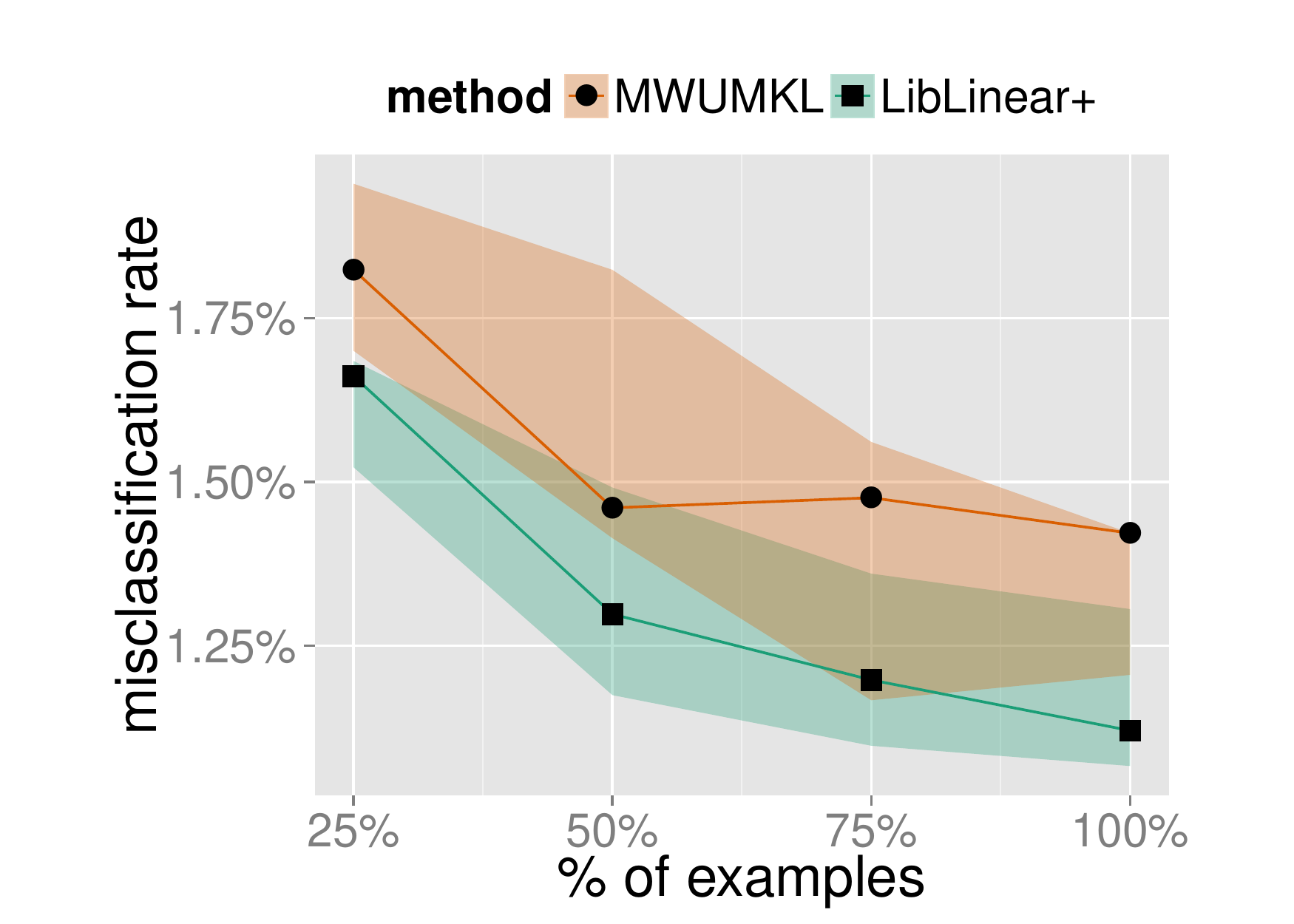}
  	\label{fig:web-error}
	}
	\vspace{-0.25in}
  \caption{\footnotesize \adult ($n=48842$, $d=123$) and \web ($n=64700$, $d=300$) with $m=12$ kernels}
\end{figure}

\paragraph{Kernel Scalability.}
We aim to demonstrate not only that \mwumkl performs well with the number of examples, but also that it performs well against the number of kernels.
In fact, for an MKL algorithm to be truly scalable it should do well against \emph{both} examples and kernels.

For kernel scalability, we present the training times for the best parameters of several of the datasets, divided by the number of kernels used, versus the size of the dataset (see Figure~\ref{fig:mix-time-ker-size}). 
We divide time by number of kernels because time scales very close to linearly with the number of kernels for all methods.
Also presented are log-log models fit to the data, and the median of each experiment is plotted as a point.

We report the time for the same experiments that produced the results in Figure~\ref{fig:mix-small-error}, and also train on increasing proportions of \mush ($1625$, $3250$, $4875$, and $6500$ examples) with $1344$ per-feature kernels.
With these selections, we are testing $mn$ in the neighborhood of $8.7$ million elements.

\begin{figure}[!htbp]
  \centering
  \includegraphics[width=0.8\columnwidth]{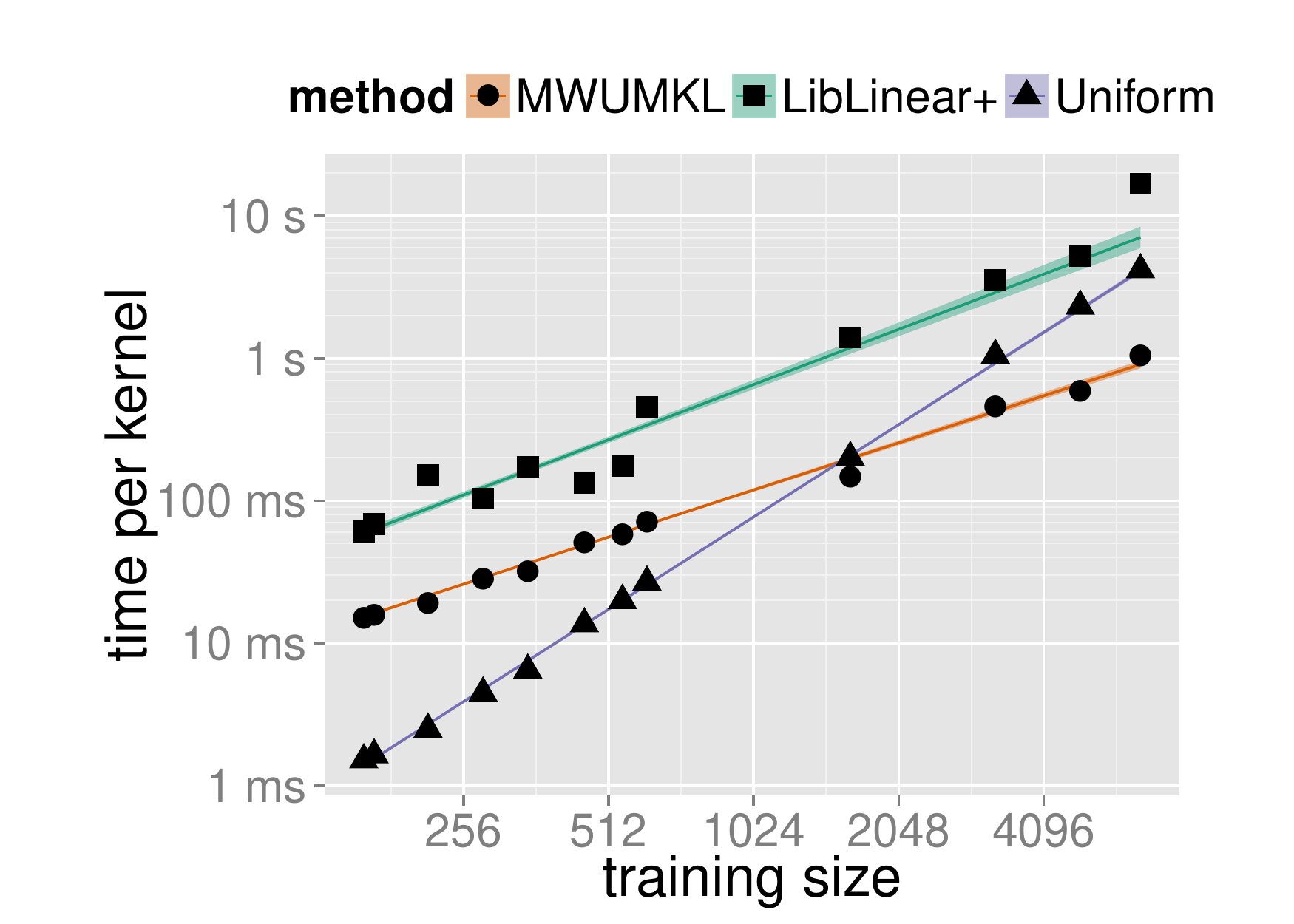}
  \caption{\footnotesize Time per kernel vs. data size for small and medium data sets (log-log).}
  \label{fig:mix-time-ker-size}
\end{figure}

As expected, \uni scales quadratically or more with the number of examples, performing very well at the lower range. 
The number of examples from \mush is not so high that LibSVM runs out of memory, but we do see the algorithm's typical scaling.

\llplus shows slightly superlinear scaling, with a high multiplier due to the matrix computations required for the feature transformations.
As we run the algorithm on \mush, the number of samples taken for the kernel approximations is reduced so that the features can fit in machine memory. 
Even so, this reduction doesn't offer any help to the scaling and at $6500$ examples with $1344$ kernels, training time is several hours.

Even though we reduced the number of samples for \llplus, \mwumkl outperforms both \uni and \llplus when both examples and kernels are greater than about $10^3$.

\paragraph{Dynamic Kernels.}
We also present results for a few datasets with lots of kernels.
By computing columns of the kernel matrices on demand, we can run with a memory footprint of $O(mn)$, improving scalability without affecting solution quality (a technique also used in \smomkl \cite{DBLP:conf/nips/VishwanathansAV10}). 
Table \ref{tab:res-dyn-ker} shows that we can indeed scale well beyond tens of thousands of points, as well as many kernels. 

\begin{table}[!htbp]
  \footnotesize
  \centering
  \begin{tabular}{|c|c|c|c|}
    \hline
    \textbf{Dataset}       & \textbf{\#Points}  & \textbf{\#Kernels}    & \textbf{Time}   \\
    \hline
		\adult     & 39073     &  3           &  $13$ minutes      \\
		\cod       & 47628     &  3           &  $147$ seconds      \\
		\sonar1M   & 208       &  1000000     &  $3.65$ hours      \\
    \hline
  \end{tabular}   
  \caption{\footnotesize \mwumkl with on-the-fly kernel computations.}
  \label{tab:res-dyn-ker} 
\end{table}

We choose the above datasets to compare against another work on scalable MKL~\cite{DBLP:conf/kdd/JainVV12}. 
\citet{DBLP:conf/kdd/JainVV12} indicate the ability to deal with millions of kernels, but in effect the technique also has a memory footprint of $\Omega(mn)$ (the footprint of \mwumkl is $\Theta(mn)$, in contrast).
This limits any such approach to \emph{either} many kernels \emph{or} many points, but not both.

Since the work in \citet{DBLP:conf/kdd/JainVV12} does not provide accuracy numbers, a direct head-to-head comparison is difficult to make, but we can make a subjective comparison. 
The above table shows times for \mwumkl with accuracy similar to or better than what \llplus can achieve on the same datasets. 
The time numbers we achieve are similar in order of magnitude when scaled to the number of kernels demonstrated in \citet{DBLP:conf/kdd/JainVV12}.

\section{Conclusions and Future Work}
\label{sec:disc}

We have presented a simple, fast and easy to implement algorithm for multiple kernel learning.
Our proposed algorithm develops a geometric reinterpretation of kernel learning
and leverages fast MMWU-based routines to yield an efficient learning algorithm.
Detailed empirical results on data scalability, kernel scalability and with dynamic kernels demonstrate that we are significantly faster than existing legacy MKL implementations and outpeform \llplus as well as \uni.

Our current results are for a single machine. As mentioned earlier, one of our
future goals is to add parallellization techniques to improve the scalability of
\mwumkl over data sets that are large and use a large number of kernels. 
The \mwumkl algorithm lends itself easily to the \emph{bulk synchronous parallel} (BSP) framework~\cite{Valiant:1990:BMP:79173.79181}, as most of the work is done in the loop that updates $\vec{G}\alpha$ (see the last line of the loop in Algorithm~\ref{alg:mwu-qcqp}). 
This task can be ``sharded'' for either kernels or data points, and scalability of $O(mn)$ would not suffer under BSP.
Since there are many BSP frameworks and tools in use today, this is a natural direction to experiment.

\section{Acknowledgments}
This research was partially supported by the NSF under grant CCF-0953066. 
The authors would also like to thank Satyen Kale and S\'{e}bastien Bubeck for their valuable feedback.

\clearpage
{\small
\bibliography{mkl}
\bibliographystyle{plainnat}
}

\end{document}